\documentclass[letterpaper, 10 pt, conference]{ieeeconf}  % Comment this line out
\IEEEoverridecommandlockouts                              % This command is only
\usepackage{subfigure}

\usepackage{setspace} 
\usepackage{graphicx}
\usepackage{epstopdf}
\usepackage{epsfig} % for postscript graphics files
\usepackage{mathptmx} % assumes new font selection scheme installed
\usepackage{times} % assumes new font selection scheme installed
\usepackage{amsmath} % assumes amsmath package installed
\usepackage{amssymb}  % assumes amsmath package installed
\usepackage{tabularx}
\usepackage{dcolumn}
\usepackage{color} 
\usepackage{algorithm} 
\usepackage{algorithmic}  
\usepackage[algo2e,linesnumbered,vlined,ruled]{algorithm2e} 
\usepackage{balance}
\usepackage{accents}
\usepackage{resizegather}
\usepackage{theorem} 
\usepackage{mathtools}
\DeclareMathOperator{\ap}{\emph{A}\MRkern \emph{P}}
\newcommand{\MRkern}{%
  \mkern-3mu
  \mathchoice{}{}{\mkern0.2mu}{\mkern0.5mu}%
}

\newtheorem{definition}{\bf{Definition}}
\newtheorem{assumption}{\bf{Assumption}}

\newtheorem{problem}{\bf{Problem}}
\newtheorem{remark}{\bf{Remark}}
\newtheorem{theorem}{\bf{Theorem}}
\newtheorem{lemma}{Lemma}

\newcommand{\norm}[1]{\left\lVert#1\right\rVert}

\newcommand{\FP}{\mathcal{P}}
\newcommand{\FA}{\mathcal{A}}

\title{\LARGE \bf
 Probabilistically Guaranteed Satisfaction of Temporal Logic Constraints During Reinforcement Learning}

\author{Derya Aksaray$^1$, Yasin Yaz{\i}c{\i}o\u{g}lu$^2$, and Ahmet Semi Asarkaya$^1$ %<-this % stops a space
\thanks{$^1$D. Aksaray and A.S. Asarkaya are with the Department of Aerospace Engineering and Mechanics, University of Minnesota, Minneapolis, MN, 55455, {\tt\small daksaray@umn.edu, asark001@umn.edu}}
\thanks{$^2$Y. Yaz{\i}c{\i}o\u{g}lu is with the Department of Electrical and Computer Engineering, University of Minnesota, Minneapolis, MN, 55455, {\tt\small ayasin@umn.edu}.}
}

\begin{document}
\bibliographystyle{abbrv}
\maketitle

\begin{abstract}
We propose a novel constrained reinforcement learning method \color{black} for finding optimal policies in  Markov Decision Processes while satisfying temporal logic constraints with a desired probability throughout the learning process. 
An automata-theoretic approach is proposed to ensure the probabilistic satisfaction of the constraint in each episode, which is different from penalizing violations to achieve constraint satisfaction after a sufficiently large number of episodes. The proposed approach is based on computing a lower bound on the probability of constraint satisfaction and adjusting the exploration behavior as needed. We present theoretical results on the probabilistic constraint satisfaction achieved by the proposed approach. We also numerically demonstrate the proposed idea in a drone scenario, where  the constraint is to perform periodically arriving pick-up and delivery tasks and the objective is to fly over high-reward zones to simultaneously perform aerial monitoring.

%We present a novel approach to reinforcement learning, where a temporal logic specification expressing a persistent mission is aimed to be satisfied throughout the learning process. In real-life applications of autonomous systems and robots, specifications implying a persistent mission may encode safety (e.g., never run out of fuel) or a critical mission requirement (e.g., periodically visit a specific location). These specifications are required to be repetitively satisfied during the operation. In this paper, we benefit from rich expressiveness of the temporal logics to express a given task to be persistently accomplished. We propose a reinforcement learning algorithm that maximizes the expected sum of rewards while ensuring probabilistic guarantees on the satisfaction of the desired task in every episode during learning. We prove the correctness of the proposed algorithm and show simulation results.
\end{abstract}

\section{Introduction}
%Markov Decision Process (MDP) framework has been frequently used in the literature when solving stochastic decision-making and planning problems (e.g., \cite{papadimitriou1987,feinberg2012}). In cases where there is an unknown element in the MDP such as unknown reward or unknown transition probabilities, 
Reinforcement learning (RL) has been widely used to learn optimal control policies for Markov Decision Processes (MDPs) through trial and error \cite{sutton2018}. When the system is also subject to constraints, traditional RL algorithms can be used to learn optimal feasible solutions after a sufficiently long amount of training by severely penalizing infeasible trajectories. However, this approach does not provide any formal guarantees on constraint satisfaction during the early stages of the learning process. Hence, this is not a viable approach for many real-life applications where the constraint violations during training may have severe consequences. 
%For example, in an aerial monitoring mission, a drone may need to learn the most informative paths to take while also ensuring that it never runs out of fuel or collides with other objects. In such a scenario, any violation of these constraints would have severe consequences. 

Safe reinforcement learning is the process of learning optimal policies while ensuring a reasonable performance or safety during the learning process \cite{garcia2015}. One prominent way to achieve safe RL is modifying the exploration process so that the system selects actions while avoiding unsafe configurations. For example, prior information or transfer learning ideas can be used to reduce the time spent with random actions during exploration (e.g., \cite{driessens2004,abbeel2005}), which typically do no have theoretical guarantees. Learning can also be achieved over a constrained MDP (e.g., \cite{efroni2020}), where the goal is to maximize the expected sum of reward subject to the expected sum of cost being smaller than a threshold. Such works typically have theoretical guarantees such as bounded regret on the performance and constraint violation. Control-barrier functions or Hamilton-Jacobi methods can also be adopted to stay inside a safe region during learning (e.g., \cite{cheng2019,fisac2018}). In these approaches, safety is mainly understood as visiting ``good states" and avoiding ``bad states". However, complex missions typically involve constraints on not only the current state but also the system's trajectory. For example, suppose that a robot must visit first region A and then region B. Regions A and B may not be categorized as safe or unsafe, but visiting B before visiting A may imply a mission failure. 
Temporal logics (TL) \cite{baier2008} provide a powerful way of describing such complex spatial and temporal specifications. Some studies in the literature address the RL problem under TL constraints. For example, Linear Temporal Logic (LTL) constraints are considered in a model-free learning framework and maximum possible LTL constraint satisfaction is achieved in \cite{hasan2020}. A reactive system called shield is proposed in \cite{alshiekh2018} which corrects the chosen action if it causes the violation of an LTL specification.  

In this paper, we introduce a reinforcement learning algorithm for maximizing the expected sum of rewards subject to the satisfaction of a TL constraint during the learning process with some desired probability. The decision-making of an agent is modeled as an MDP, and the constraint is expressed using a bounded TL which is encoded as a finite state automaton. We construct a time-product MDP and formulate the learning problem over the time-product MDP. We prove that the proposed learning algorithm enables the agent to satisfy the TL constraint in each episode with a probability greater than a desired threshold.

%while learning to maximize the cumulative expected reward. 

%This paper is different than the existing works addressing RL under TL constraints as follows: 1) our algorithm gets a desired probability of constraint satisfaction , i.e., $Pr_{des}$, as an input, and 2) we ensure that the probability of TL satisfaction during each episode of learning is at least $Pr_{des}$

This paper differs from the existing works on RL under TL
constraints (e.g., \cite{hasan2020,alshiekh2018}) as follows: 1) we consider bounded TL constraints with explicit time parameters, which are richer than LTL
constraints, 2) we ensure that the probability of TL
satisfaction in each episode of learning is not below a desired
threshold,  $Pr_{des}$.
%This paper is different than the existing works addressing RL under TL constraints, which typically focus on maximizing the probability of constraint satisfaction (E>G> REFS). The proposed algorithm ensures that the probability of TL satisfaction during each episode of learning is at least $Pr_{des}$, which is an input to the algorithm.\color{black}  
Note that this formulation is more flexible than enforcing the maximum probability of satisfaction (e.g., \cite{hasan2020,alshiekh2018}) since it allows the user to tune the performance by selecting $Pr_{des}$ based on the trade-off between risk (constraint violation) and efficiency (reward collection). To the best of our knowledge, this is the first study that addresses a constrained reinforcement learning problem, where the goal is to maximize expected reward while satisfying a bounded temporal logic constraint with a probability greater than a desired threshold throughout the learning process (even in the first episode).

\section{Preliminaries: Temporal Logic}
%\subsection{Overview}
Temporal logic (TL) is a mathematical formalism to reason about the behavior of a system in terms of time. %Temporal logic specifications include various atomic propositions that can be connected with temporal and Boolean operators to express complex behaviors. For example, consider a drone that visits a set of regions, and let A be an atomic proposition having a truth value, i.e., A  is true if the drone is visiting the corresponding region and it is false otherwise. One can define various temporal logic specifications including this atomic proposition. For instance, satisfying ``eventually A" implies that there exists a time instant region A is visited; or satisfying ``always A" implies that region A is visited all the time; or satisfying ``always eventually A" implies that region A is visited periodically.  
There are various TLs, and one way to categorize them is based on the length of words they can deal with. Let $\ap$ be a set of atomic propositions each of which has a truth value. For example, let $A \in \ap$ be an atomic proposition. One can say $A$ is true if a drone is monitoring region $A$ and it is false otherwise. A word is a sequence of elements from $\ap$. In this regard, there are TLs like Linear Temporal Logic (LTL) that can deal with words of infinite length. LTL is extensively used in various domains and has efficient off-the-shelf tools for verification and control synthesis (e.g., \cite{gerth1995,kress2009}). While LTL can express a specification such as ``eventually visit region A", it cannot capture temporal properties with explicit time constraints (e.g., ``eventually visit region A in 10 minutes").  Temporal logics such as Bounded Linear Temporal Logic (\cite{ishii2015,tkachev2013}), Interval Temporal Logic \cite{cau1997}, and Time Window Temporal Logic \cite{twtl} deal with words of finite length to overcome this limitation. For most of the bounded TLs, satisfactory cases are called accepting words. 
%These TLs can express specifications such as ``eventually from time 1 to 3, visit region A for 2 consecutive time steps" which can be satisfied in two ways: 1) region A is visited in the first and second time steps, or 2) region A is not visited in the first time step, but it is visited in the second and third time steps. These are called the accepting words. 
The set of all accepting words is called the accepting language of the TL, which can be represented as a deterministic finite state automaton (FSA) $\FA = (Q, q_{init}, 2^{\ap}, \delta, F_{\FA})$ where $Q$ is a finite set of states, $q_{init} \in Q$ is the initial state, $2^{\ap}$ is the input alphabet, $\delta : Q \times 2^{\ap} \rightarrow Q$ is the transition function, $F_{\FA} \subseteq Q$ is the set of accepting states.  
\begin{figure}[htb!]
\begin{center}
\includegraphics[width=0.85\columnwidth]{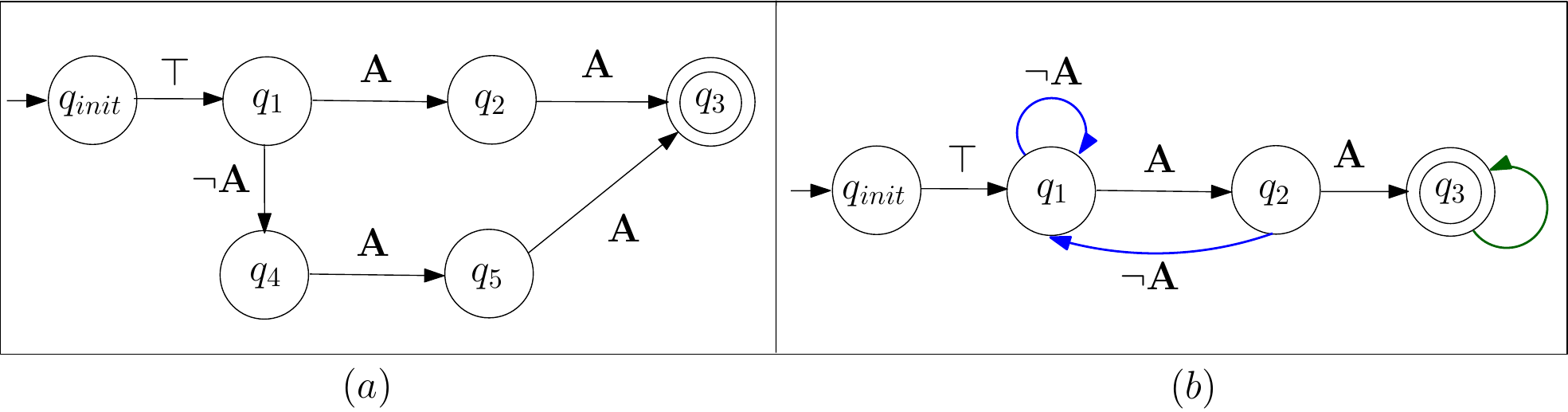}%
\caption{(a) The FSA of ``eventually from time 1 to 3, visit A for 2 consecutive time steps". $Q=\{q_{init},q_0,q_1,q_2,q_3,q_4,q_5\}$, $\ap=\{A,\neg A\}$ where $\neg A$ refers to the negation of A, $\top$ is the true constant, $\delta(q_{init},\top) = q_1$, $\delta(q_{1},A) = q_2$, $\delta(q_{1},\neg A) = q_4$, $\delta(q_4,A)=q_5$, $\delta(q_2,A) = \delta(q_5,A) = q_3$, $F_{\FA} = \{q_3\}$. Note that each path that starts from $q_{init}$ and ends at $q_3$ represents an accepting word. (b) The modified FSA of ``eventually from time 1 to 3+$\tau$, visit A for 2 consecutive time steps" for all possible $\tau$.}
\label{fig:fsa}
\end{center}
\end{figure}

 %\vspace{-4mm}
Temporal relaxation can be defined for any bounded TL specification with explicit time parameters. For instance, consider a specification as ``eventually from time 1 to 3, visit region A for 2 consecutive time steps". Its temporally relaxed version will be ``eventually from time 1 to $3+\tau$, visit region A for 2 consecutive time steps" where $\tau$ is a slack variable that can expand or shrink the time window. Such a temporal relaxation idea has been introduced for TWTL in \cite{twtl,aksaray2016dynamic}, and the authors also propose an algorithm that constructs an FSA encoding all the possible temporal relaxations of a formula. For example, Fig.~\ref{fig:fsa} illustrates the FSAs of a bounded TL formula and its temporally relaxed version. %Accordingly, an FSA for all possible temporal relaxations can be obtained by adding backward edges, which also makes the size of the modified FSA more compact as in Fig.~\ref{fig:fsa}(b). 

\color{black}
\section{Problem Statement}

\subsection{Motivation}

%In a traditional reinforcement (RL) algorithm , the goal of the agent is to learn the best policy in an environment by trial and error. %It simply learns from the experience, i.e., the agent needs to reach a state causing a violation so that it could learn that state is not satisfying the specification. 
%However, this approach might not be feasible in real-world missions. For instance, some missions require safety constraints (e.g.,  never run out of fuel). Also, there might be some other cases where the agent has to complete a persistent task (e.g.,  periodically  visit some specific regions) during the learning. Furthermore, certain tasks may have the highest priority among other interests. In that case, the agent would continually need to complete the given task while learning the environment in each episode.  The conventional RL algorithms does not offer any solutions for such systems.
%of the constraint satisfaction concerns as guaranteeing not to violate the constraints in real-life applications is the highest priority among the other interests. 
%In this paper,  we propose a novel learning algorithm that enables an agent to learn optimal policy to maximize a cumulative reward function while it persistently guarantees to complete 
%complex tasks during all training episodes. %This problem can be formulated as follows:

%\subsection{Problem Statement}
We consider an MDP, $M=(S,A,\Delta_M,R)$, where $S$ is the state-space, $A$ is the set of actions, $\Delta_M:S \times A \times S \rightarrow [0,1]$ is a probabilistic transition relation, and $R: S \times A \rightarrow \mathbb{R}$ is a reward function. Moreover, let $\ap$ be a set of atomic propositions, each of which has a truth value over the state-space. Let ${l:S \rightarrow 2^{\ap}}$ be a labeling function, which maps every ${s \in S}$ to the set of atomic propositions that are true when the system is in state $s$. An example of an MDP, a set of atomic propositions, and a labeling function is shown in Fig.~\ref{fig:mdp}.

\begin{figure}[htb!]
\begin{center}
\resizebox*{0.6\columnwidth}{!}{\includegraphics{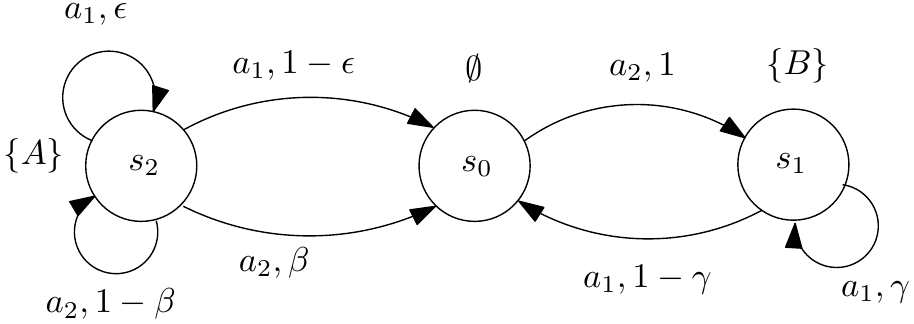}}%
\caption{An MDP where $S=\{s_0,s_1, s_2\}$, $A=\{a_1,a_2\}$, $\ap=\{A,B\}$, $l(s_{2})=\{A\}$, $l(s_0) = \emptyset$, $l(s_1)=\{B\}$. Labels on each edge indicate the corresponding action and transition probability. $\epsilon,\beta,\gamma \in (0,1)$.}
\label{fig:mdp}
 % %\vspace{-6mm}
\end{center}
\end{figure}

%For example, let $\ap=\{delivery-zone\}$, and let $\{s_1,s_2,s_5,s_6\}\subset S$ be the states where the atomic proposition $delivery-zone$ is true. Then, $l(s_1)=l(s_2)=l(s_5)=l(s_6)=\{delivery-zone\}$.

%\textcolor{red}{DA: S comprises a set of partitions over a continuous space. An illustration.}

Given an MDP, $\pi:S\rightarrow A$ is called a policy. A policy $\pi$ is stationary if it does not change over time. In traditional RL, the probabilistic transition relation $\Delta_M$ is assumed to be unknown, and the agent is required to find an optimal control policy $\pi^*$ that maximizes the expected sum of rewards, i.e., $E^{\pi}\Big[\sum_{k=0}^T r_{k+t+1}\Big] \quad \text{ or } \quad E^{\pi}\Big[\sum_{k=0}^\infty \gamma^k r_{k+t+1}\Big]$
% \begin{equation}
% \label{eq:Qobj}
% E\Big[\sum_{k=0}^T r(s_{k+t+1})\Big] \quad \text{ or } \quad E\Big[\sum_{k=0}^\infty \gamma^k r(s_{k+t+1})\Big], 
% \end{equation}
where $r_t$ is the reward obtained at time $t$, and $\gamma \in (0,1]$ is the discount factor. In the literature, various learning algorithms, e.g., Q-learning \cite{watkins1992}, were shown to find the optimal policy. 

In many real-life problems, the agent may also have constraints that should be satisfied during learning. For example, some safety constraints can be enforced by redefining the state-space (e.g., $S^{\prime} = S \setminus S_c$ where $S_c \subset S$ is the set of unsafe states and learning over $S^{\prime}$). However, not every constraint can be easily satisfied by removing the violating states from the state-space. In some cases, the agent may be required to follow a trajectory that satisfies a complex specification throughout the learning process. For instance, consider a drone whose primary task is achieving a pick up and delivery task arriving periodically. Moreover, the drone can have a secondary task of maximizing situational awareness via aerial monitoring (e.g., traffic, infrastructure, or environmental monitoring). In this scenario, the primary task can be considered as the constraint of the drone, and the secondary task can be formulated as an RL problem where the reward at each state represents the value of information that can be collected from the corresponding location. Overall, the objective becomes to learn a policy that maximizes the expected cumulative reward subject to the pick up and delivery constraint that should always be completed during the learning process. 

Note that learning a policy to satisfy a TL specification can be achieved via RL (e.g., \cite{aksaray2016,sadigh2014,bozkurt2020}). However, such methods do not guarantee the satisfaction of the TL specification throughout the learning process. This paper proposes a constrained RL algorithm with a probabilistic guarantee on the satisfaction of TL constraints in every episode of learning, which is different from learning to satisfy the constraints after sufficient training. Ensuring the desired probability of constraint satisfaction throughout learning has two critical advantages: 
1) In the proposed approach, having a finite learning horizon may cause suboptimality but not infeasibility.
2) The proposed approach can be used to improve an initial policy during the mission while always maintaining feasibility, which is a crucial capability since violating the constraints may have severe consequences. 

\subsection{Temporal Logics As Constraints To Learning Problems}
%FSA'DA TRACK ETME NIYE ONEMLI?
The FSA of a bounded TL specification allows the tracking of progress toward satisfaction by compactly encoding all the accepting words. However, RL problems are typically posed in a stochastic setting where the actions may not always result in desired transitions that can lead to the violation of TL. Thus, a total FSA, which either accepts or rejects a word \cite{burgin2013}, is needed in learning problems to track both satisfying and violating cases. 
Alternatively, a temporal relaxation of the bounded TL can be used in learning problems since the backward edges in the FSA of a temporally relaxed formula also encode possible regression in the task satisfaction.

%, the FSA needs to be total, . State transitions after a TL violation can be tracked over an FSA if some modifications are done to its states and transitions. For instance, in order to encode the violations as well as the satisfaction cases, an additional state with a self-loop can be included and edges from other states to the new state can be added to capture the violation cases. 
%Such a representation can still exactly encode all accepting words of the TL. 
%BIR FARKI VAR AMA, ARTIK HER PATH SUCCESS DEGIL. 

Furthermore, representing the transitions after the TL satisfaction is also important in learning problems because learning can continue if there is remaining time in the episode. To this end, a self loop needs to be added to the accepting state of the FSA (e.g., a self-loop to the state $q_3$ in Fig.~\ref{fig:fsa}(b)). Throughout this paper, when we say the FSA of a TL formula $\phi$, we consider the FSA\footnote{An algorithm for the construction of the FSA for a temporally relaxed TWTL can be found in \cite{twtl}.} corresponding to the temporal relaxation of the formula, i.e., $\phi(\tau)$, with a self-loop added to the accepting state.  
% \begin{figure}[htb!]
% \begin{center}
% \resizebox*{0.8\columnwidth}{!}{\includegraphics{fsa2.pdf}}%
% \caption{a) The modified FSA (removing states $q_4$ and $q_5$ and adding blue edges) of ``eventually from time 1 to 3+$\tau$, visit region A for 2 consecutive time steps" for all possible $\tau$. b) The addition of green edge enables to track state transitions even after the satisfaction.}
% \label{fig:relaxed_fsa}
% \end{center}
% \end{figure}
Now, we formally define the RL problem with constraint satisfaction during learning. 
% % %\vspace{-2mm}
\begin{problem} 
Given an MDP $M=(S,A,\Delta_M,R)$ with unknown transition probability function $\Delta_M$ and unknown reward function $R$, a set of atomic propositions $\ap$, a labeling function $l:S \rightarrow 2^\emph{$\ap$}$, let $\phi$ be a TL constraint that needs to be satisfied periodically during learning. Given a desired probability of satisfaction $Pr_{des} \in [0,1)$, learn the optimal control policy
\begin{equation}
\label{pistareq}
\pi^* = \arg\max\limits_{\pi } E^{\pi}\Big[\sum_{t=0}^\infty \gamma^t r_t \Big] 
\end{equation}
such that, in each episode $j$ of the learning process,
\begin{equation}
\label{eq:Qobj}
\begin{aligned}
\Pr\big(\mathbf{o}(j(T+1)),\mathbf{o}(j(T+1)+1),\dots,\mathbf{o}(j(T+1)+T) \models \phi{(\tau_j)}\big) \color{black} &\geq Pr_{des}, \quad \forall j\geq0\\
\norm{\phi(\tau_j)} &\leq T,
\end{aligned}
\end{equation}
where $s_t$ is the state at time $t$, $T=\norm{\phi(0)}$ is the time bound\footnote{The time bound of $\phi$ is the maximum time needed to satisfy it \cite{twtl}.} of $\phi$ with no relaxation, $\mathbf{o}(j(T+1)),\mathbf{o}(j(T+1)+1),\dots,\mathbf{o}(j(T+1)+T)$ is an output word based on the state sequence $s_0,s_1,s_2,\dots, s_T$ over the MDP (e.g., $\mathbf{o}(j(T+1)) = l(s_{j(T+1)})$), $\norm{\phi(\tau_j)}$ is the time bound of the relaxed formula $\phi(\tau_j)$, and $\gamma \in (0,1]$ is the discount factor. 
\label{problem}
\end{problem}

%if the reward is designed in a manner that captures both the information in the environment and the pick up and delivery task, then a standard reinforcement learning algorithm can be utilized to learn a policy that maximizes the sum of reward as well as satisfying the constraint. However, while such a policy can be learned, there is no guarantee on the constraint satisfaction during learning. This paper is different than the state-of-the art techniques by providing constraint satisfaction guarantee during every episode of the learning process. Moreover, the proposed idea can accommodate rich and complex specifications (including spatial, temporal, and logical relations) that can be expressed by bounded temporal logics with automaton representation. %\footnote{Signal Temporal Logic is also a bounded temporal logic but it does not have an automaton representation}.   

\section{Proposed Approach}

Learning with TL objectives or constraints cannot be achieved over a standard MDP which contains only the agent's current state. Suppose that an agent can move in both cardinal and ordinal directions on a grid as shown in Fig.~\ref{fig:productmdp-mot}. Let the task be ``eventually visit A and then B". Given the agent's current state shown in triangle, it must learn to select 1) the action shown in green arrow if region A has not been visited yet, and 2) the action shown in red arrow if region A has been visited before. Overall, the agent needs to know both its current state and the task status to decide what to do next. An automaton constructed from the TL specification naturally keeps track of the task's progress. Hence, a typical approach to encode both the state transitions over an MDP and the task progress is to construct a product MDP.  

\begin{figure}[htb]
    \centering
    \includegraphics[width=0.4\columnwidth, trim={0 2.85cm 0 0cm},clip]{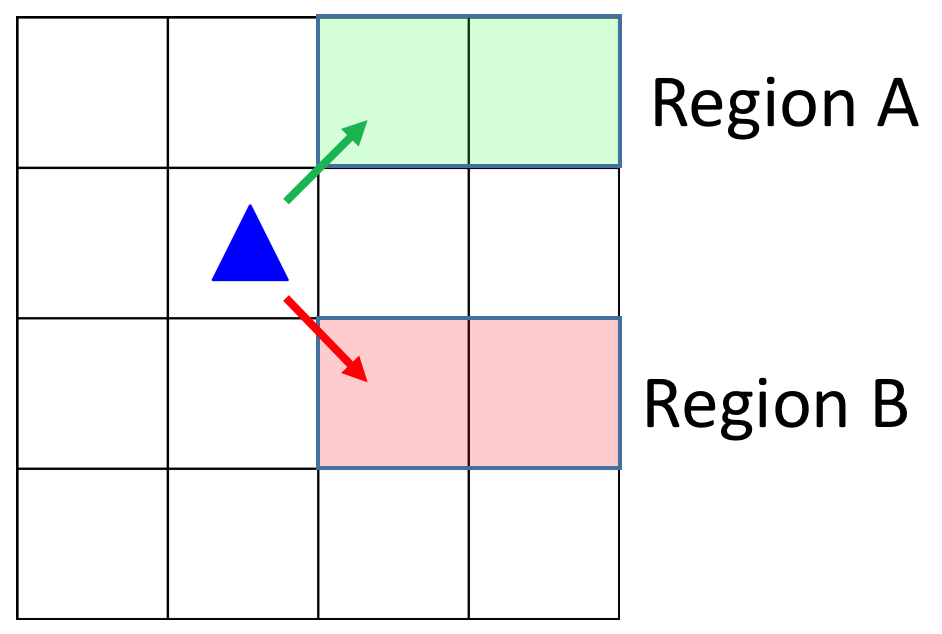}
    \caption{An illustration of an agent (shown in triangle) and a task such as "eventually visit $A$ and then visit $B$". %If $A$ is visited before, the optimal action selection is going to $B$. However, if $A$ has not been visited yet, then the optimal action selection is going to $A$. Optimal action selections are not the same, if the same state is occupied but task progresses are different.
    }
    \label{fig:productmdp-mot}
     % %\vspace{-0.5cm}
\end{figure}

\begin{definition} (Product MDP)
Given an MDP $M$, a set of atomic propositions \emph{$\ap$}, a labeling function $l:S \rightarrow 2^\emph{$\ap$}$, and an FSA $\FA$, a product MDP is a tuple $\FP =M \times \FA = (S_{\FP},P_{init},A,\Delta_\FP, R_{\FP}, F_{\FP})$, where

\begin{itemize}
\item $S_\FP = S \times Q$ \textit{is a finite set of states;}
\item $P_{init}=S \times \{q_{init}\} \subseteq S_\FP$ \textit{is the set of initial states;}
\item $A$ \textit{is the set of actions;}
\item $\Delta_\FP : S_\FP \times A \times S_\FP \rightarrow [0,1]$ \textit{is the probabilistic transition relation such that for any two states, $p=(s,q) \in S_\FP$ and $p^{\prime}=(s^{\prime},q^{\prime}) \in S_\FP$, and any action $a \in A$, $\Delta_\FP(p,a,p^\prime)=\Delta_M(s,a,s^\prime)$ and $\delta(q,l(s))=q^\prime$}; 
\item $R_{\FP}:S_\FP \times A \rightarrow \mathbb{R}$ \textit{is the reward function such that $R_{\mathcal{P}}(p,a) = R(s,a)$ for $p=(s,q) \in S_\FP$;}
\item $F_\FP = (S \times F_\FA) \subseteq S_\FP$ \textit{is the set of accepting states.}
\end{itemize}
\end{definition}

In Problem~\ref{problem}, the time bound of the TL constraint $\phi$ determines the episode length, and the remaining episode time is also critical to the optimal action selection. Again, consider the example in Fig.~\ref{fig:productmdp-mot}. This time, let ``eventually visit A and then visit B" be a constraint during learning. If the agent shown in triangle has not satisfied this constraint yet and the episode has just started, it can either explore the environment by selecting an admissible action or make progress towards the satisfaction of the constraint by selecting the green arrow. Now, suppose that the agent is at the same state, the constraint has not been satisfied yet, but the episode is about to finish. In that case, the agent must pick the green arrow for constraint satisfaction. Overall, learning under a bounded TL constraint needs to be achieved over a space that encodes the physical state, the automaton state, and the remaining episode time. %Hence, we define a time-product MDP.   

% % %\vspace{-3mm}
\begin{definition} (Time-Product MDP)
Given a product MPD $\mathcal{P}$ and a time set $\mathcal{T}=\{0,\dots,T\}$, a time-product MPD is a tuple $\mathcal{P}^{\mathcal{T}} = \mathcal{P} \times \mathcal{T} =(S^{\mathcal{T}}_{\mathcal{P}}, P^{\mathcal{T}}_{init}, A, \Delta^{\mathcal{T}}_\FP, R^{\mathcal{T}}_{\mathcal{P}}, F^{\mathcal{T}}_{\mathcal{P}})$ where,
\begin{itemize}
\item $S^{\mathcal{T}}_{\mathcal{P}} = S_{\mathcal{P}} \times \mathcal{T}$ is a finite set of states;

%\item $p^t=(p,t) \in S^{\mathcal{T}}_{\mathcal{P}}$ is a time product MDP state at the time instance $t \in \mathcal{T} $;

\item $P^{\mathcal{T}}_{init} = P_{init} \times \{0\} \subseteq S^{\mathcal{T}}_{\mathcal{P}}$ is the set of initial states;% where $ p^{\mathcal{T}}_{init} \subseteq S^{\mathcal{T}}_{\mathcal{P}}$;
\item $A$ is the set of actions;
\item $\Delta_\FP^{\mathcal{T}} : S^{\mathcal{T}}_{\mathcal{P}} \times A \times S^{\mathcal{T}}_{\mathcal{P}} \mapsto [0,1]$ is the probabilistic transition relation such that $\Delta_\FP^{\mathcal{T}}(p_i^t,a,p_j^{t+1})=\Delta_\FP(p_i,a,p_j)$ for an action $a \in A$ and two time-product MDP states $p_i^t=(p_i,t) \in S^{\mathcal{T}}_{\mathcal{P}}$ and $p_j^{t+1}=(p_j,t+1) \in S^{\mathcal{T}}_{\mathcal{P}}$;
\item $R^{\mathcal{T}}_{\mathcal{P}}  : S^{\mathcal{T}}_{\mathcal{P}} \times A \mapsto \mathbb{R}$ is the reward function such that $R^{\mathcal{T}}_{\mathcal{P}}(p^t ,a)=R_{\mathcal{P}}(p,a)$ and $p^t=(p,t) \in S^{\mathcal{T}}_{\mathcal{P}}$;
\item $F^{\mathcal{T}}_{\mathcal{P}} = (F_{\mathcal{P}} \times \mathcal{T}) \subseteq S^{\mathcal{T}}_{\mathcal{P}}$ is the set of accepting states.
\end{itemize}
\label{def:tmdp}
\end{definition}

In this paper, we propose a solution to Problem~\ref{problem} by introducing a modified Q-learning algorithm which is executed over a time-product MDP and ensures probabilistic guarantees for TL constraint satisfaction while learning the optimal policy. The proposed algorithm is different than standard Q-learning in terms of the action selection at each state. In standard Q-learning, actions are selected based on only exploration-exploitation considerations. In our problem, we want to achieve probabilistic constraint satisfaction guarantees in each episode. To this end, we derive an equation that can evaluate the worst case probability of reaching an accepting state from the current state on the time-product MDP. In cases where the worst case satisfaction probability of the next state becomes smaller than the desired probability, an action is selected to minimize the distance to the accepting states. Before explaining the details of the algorithm, we make some mild assumptions and introduce a few definitions.

 % %\vspace{-2mm}
\begin{assumption}
\label{assm}
  Given some $\epsilon \in [0,1)$, for each state $s$ and action $a$ of the MDP, the states $s^{'}$ such that $\Delta_{M}(s,a,s^{\prime})> 0$ and the states $s^{\prime \prime}$ such that $\Delta_{M}(s,a,s^{\prime \prime})\geq 1- \epsilon$ are known.
\end{assumption}

%\begin{assumption}
%\label{assm}
%Given an MDP and some $\epsilon \in [0,1)$, we assume that any feasible transition can be reverted in one time step with probability at least $1-\epsilon$, i.e., for any $s,s' \in S$
%\begin{equation}
%\label{asseq}
%\resizebox{0.91\columnwidth}{!}{
%\exists a \in A: \Delta_M(s,a,s^\prime) > 0 \Rightarrow \exists a^\prime \in A: \Delta_M(s^\prime,a^\prime,s) \geq 1-\epsilon.
%}
%\end{equation}
%Furthermore, for each state $s$ and $a$, the set of states $s^'$ such that $\Delta_M(s,a,s^\prime)> 0$ and the set of states $s^{''}$ such that $\Delta_M(s,a,s^{\prime\prime})\geq1-\epsilon$ are known.
%\label{reversible}
%\end{assumption}

Note that Assumption \ref{assm} does not require knowing the actual transition probabilities. Instead, it only requires knowing which transitions are feasible and which of those feasible transitions are sufficiently likely to occur for each state-action pair. For example, for a mobile agent as in Fig. \ref{fig:productmdp-mot}, suppose that each action (e.g., ``move up") results in moving to the desired cell with probability $0.9$ or moving to one of the other adjacent cells with probability $0.1$. While the actual values of these probabilities are unknown, some prior information (empirical data) may indicate that for each action the only transition that occurs with probability at least $0.7$ ($\epsilon=0.3$) is moving to the desired cell. 

%In the same example, the assumption also implies that any feasible transition can also be reverted with probability at least $0.7$ by choosing to move back to the previous location in the next time step. \textcolor{blue}{For the simplicity of the expressions derived in this paper, we will assume such one-step reversibility, which is satisfied in many real high-level planning problems (e.g., a mobile robot can usually go back to its previous location in the next time step). However, the results of this paper can also be extended by considering $\alpha$-step reversibility (each feasible transition can be reverted via a sequence of at most $\alpha$ likely transitions for some positive integer $\alpha$) to accommodate more generalized cases.} 

\color{black}
%but we assume that a conservative lower bound is known (e.g., the transition probability is at least ).  
%JUSTIFY THE ASSUMPTION

% \begin{assumption}
% Given an MDP and some $\epsilon \in [0,1)$, we assume that, for each $s \in S$, the set of states reachable from $s$ in one time step with probability at least $1-\epsilon$, i.e.,
% \begin{equation}\small
% N(s) = \{s^\prime \mid \exists a \in A: P(s,a,s^\prime) \geq 1-\epsilon\},    
% \end{equation}
% is non-empty. We call the set $N(s)$ as the neighborhood of $s$ under \textbf{$\epsilon$-filtered transitions}. 
% \label{1-eps_known}
% \end{assumption}

% \begin{assumption}
% Given an MDP and some $\epsilon \in [0,1)$, we assume that any feasible transition can be reverted in one time step with probability at least $1-\epsilon$, i.e., for any $s,s' \in S$
% \begin{equation}
% \exists a \in A: P(s,a,s^\prime) > 0 \Rightarrow \exists a^\prime \in A: P(s^\prime,a^\prime,s) \geq 1-\epsilon. 
% \end{equation}
% \label{reversible}
% \end{assumption}

% %Similar to $N(s)$, in a product MDP, let 

% Now, we will define the \emph{distance under $\epsilon-$filtered transitions} between two product MDP states.

%Similar to $N(s)$, in a product MDP, let 

% % %\vspace{-3mm}
\begin{definition}[$\epsilon-$Stochastic Transition] 
For any time-product MDP and $\epsilon \in [0,1)$, we say that $(p_i^t,a,p_j^{t+1})$ is an $\epsilon-$stochastic transition if the probability of such a transition is at least $1-\epsilon$, i.e., $\Delta_\FP^{\mathcal{T}}(p_i^t,a,p_j^{t+1}) \geq 1-\epsilon$.
\label{e-filt-trans}
\end{definition}

As per the definition of $\epsilon-$stochastic transitions, a $0$-stochastic transition $(p_i^t,a,p_j^{t+1})$ is one that occurs with probability 1 if action $a$ is taken at state $p_i^t$. At the other extreme, as $\epsilon$ approaches 1, any feasible transition $(p_i^t,a,p_j^{t+1})$ becomes an  $\epsilon$-stochastic transition. Next, we will define the \emph{distance under $\epsilon-$stochastic transitions}. Note that there exist similar distance definitions in the literature (e.g., \cite{bisoffi2018,ulusoy2014,kantaros2020stylus}) for control synthesis problems. In this paper, we will use the distance under $\epsilon-$stochastic transitions to derive a novel lower bound on the probability of reaching an accepting state within a desired time.

% %\vspace{-3mm}
\begin{definition}[Distance under $\epsilon-$Stochastic Transitions] 
For any two time-product MDP states $p_i^t,p_j^{t+\Delta t} \in S^{\mathcal{T}}_\FP$ and $\epsilon \in [0,1)$, 
the distance between these two states under $\epsilon-$stochastic transitions is $dist^{\epsilon}(p_i^t,p_j^{t+\Delta t}) = \Delta t$ if there exists a path from $p_i^t$ to $p_j^{t+\Delta t}$ under  $\epsilon-$stochastic transitions. Also, for each $p_i^t \in S^{\mathcal{T}}_\FP$, $N_\epsilon(p_i^t)$ denotes the neighborhood of $p_i^t=(p_i,t)$ under $\epsilon-$stochastic transitions, i.e.,
\begin{equation}
\small
N_\epsilon(p_i^t) = \{(p_j,t+1) \mid \exists a \in A: \Delta_\FP(p_i,a,p_j) \geq 1-\epsilon \},  %\delta(q,l(s^\prime)) = q' \},    
\label{eq:neighbor}
\end{equation}
which is the set of time-product MDP states $p_j^{t+1}$ that can be reachable from $p_i^t$ in one time step, i.e., $dist^{\epsilon}(p_i^t,p_j^{t+1})=1$.
\label{ML-distance}
\end{definition}

% \begin{definition}[Distance under $\epsilon-$Filtered Transitions] 
% The distance under $\epsilon-$filtered transitions between any two product MDP states, $p,p' \in S_\FP$, $dist^{\epsilon}(p,p')$, is equal to the minimum number of  $\epsilon-$filtered transitions that takes the product MDP from $p$ to $p'$. For any $p=(s,q)$, $N(p)$ denotes the neighborhood of $p$ under $\epsilon-$filtered transitions,
% \begin{equation}
% \small
% N(p) = \{p'=(s',q') \mid s' \in N(s),  \delta(q,l(s^\prime)) = q' \},    
% \end{equation}
% which is the set of states $p'$ such that $dist^{\epsilon}(p,p')=1$.
% \label{ML-distance}
% \end{definition}

%The distance between two product automaton states, i.e., $dist(p_i,p_j)$ where $p_i,p_j \in S_{\FP}$, is the number of edges in a shortest path. 
Note that 
%the distance of a product MDP state to itself under $\epsilon-$stochastic transitions is zero, i.e., $dist^{\epsilon}(p_i,p_i)=0$, if $p_i \in N(p_i)$. Moreover, 
if two time-product MDP states $p_i^t$ and $p_j^{t+1}$ are disconnected, then $dist^{\epsilon}(p_i^t,p_j^{t+1})= \infty$. The distance of a time-product MDP state $p_i^t$ to the set $X \subseteq S^{\mathcal{T}}_{\FP}$ under $\epsilon-$stochastic transitions is $d^{\epsilon}(p_i^t,X) = \min_{ p_j^\tau \in X}dist^{\epsilon}(p_i^t,p_j^{\tau})$. In the following definition, we consider the set $X$ as the set of accepting states $F_{\FP}^\mathcal{T} \subseteq S_{\FP}^{\mathcal{T}}$, and $d^{\epsilon}(p_i^t)$ will refer to the distance from $p_i^t$ to the set $F_{\FP}^\mathcal{T}$ under $\epsilon-$stochastic transitions in the remainder of the paper.   
% %\vspace{-3mm}
\begin{definition} 
[Distance-To-$F_{\FP}^{\mathcal{T}}$] For any state $p^t \in S_{\mathcal{P}}^\mathcal{T}$ of a time-product MDP, the distance of $p^t$ to the set of accepting states $F_{\FP}^\mathcal{T}$ under $\epsilon$-stochastic transitions is
\begin{equation}
d^{\epsilon}(p^t) = \min\limits_{ p_j^{\tau} \in F_{\mathcal{P}}^{\mathcal{T}}}dist^{\epsilon}(p^t,p_j^\tau).
\label{eq:distance}
\end{equation}
\label{distance}
\end{definition}

  \vspace{-10mm}
\begin{assumption}
\label{assump:distance1}
The MDP and the FSA are such that any feasible transition on the resulting time-product MDP can increase the distance-to-$F_{\FP}^{\mathcal{T}}$ by at most one.
\end{assumption}

\begin{remark}
We use Assumption \ref{assump:distance1} for the simplicity of the expressions derived in this paper, particularly the lower bound in \eqref{lower-bound}. This assumption can be relaxed/removed and Alg. 1 can be modified accordingly to obtain similar guarantees on the constraint satisfaction. For example, the maximum increase in  distance-to-$F_{\FP}^{\mathcal{T}}$ under  the feasible transitions on the time-product MDP can be used to obtain a similar yet more conservative lower bound in \eqref{lower-bound}. Accordingly, a variant of Alg. 1 can be designed to ensure the desired probability of constraint satisfaction without requiring Assumption \ref{assump:distance1}. 
\end{remark}

\begin{definition}
[Go-to-$F_{\FP}^{\mathcal{T}}$ Policy] Given any time-product MDP and $\epsilon \in [0,1)$, 
Go-to-$F_{\FP}^{\mathcal{T}}, $  $\pi^{\epsilon}_{GO}: S_{\FP}^{\mathcal{T}} \rightarrow A $, is a stationary policy over the time-product MDP $\FP^{\mathcal{T}}$ such that
\begin{equation}
\label{goeq}
\pi^{\epsilon}_{GO}(p^t) = \arg\min_{a \in A} d^{\epsilon}_{min}(p^t,a),
\end{equation}
where $d^{\epsilon}_{min}(p^t,a)$ is the smallest distance-to-$F_{\FP}^{\mathcal{T}}$ among the states that can be reached from $p^t$ via $a$ with probability at least $1-\epsilon$, i.e., $d^{\epsilon}_{min}(p^t,a)= \min\limits_{p_j^{t+1}:\Delta_\FP^{\mathcal{T}}(p^t,a,p_j^{t+1})\geq1-\epsilon} d^\epsilon(p_j^{t+1}).$
% \begin{equation}
% d^{\epsilon}_{min}(p^t,a)= \min_{p_j^{t+1}:\Delta_\FP^{\mathcal{T}}(p^t,a,p_j^{t+1})\geq1-\epsilon} d^\epsilon(p_j^{t+1}).
% \end{equation}
\end{definition}

\begin{lemma} For any $p^t \in  S_{\FP}^{\mathcal{T}}$ and integer $k\geq 0$, let ${\Pr\big(p^t \xrightarrow{k} F_\FP^{\mathcal{T}};\pi^{\epsilon}_{GO} \big)}$ be the probability of reaching the set of accepting states $F_{\FP}^\mathcal{T} \subseteq S_{\FP}^\mathcal{T}$ from $p^t$ in the next $k$ time steps under the policy $\pi^{\epsilon}_{GO}$. If $d^\epsilon(p^t) < \infty$ for every $p^t \in  S_{\FP}^{\mathcal{T}}$, then
\begin{equation}
\Pr\big(p^t \xrightarrow{k} F_\FP^{\mathcal{T}};\pi^{\epsilon}_{GO} \big) \geq \sum\limits_{i=0}^{\lfloor \frac{k-d^\epsilon(p^t)}{2}\rfloor}  \frac{k!}{(k-i)!i!} \epsilon^i (1-\epsilon)^{k-i},    
\label{lower-bound}
\end{equation}
for every state $p^t \in S_{\FP}^\mathcal{T}$ such that $k \geq d^\epsilon(p^t)$.
\label{lemma}
\end{lemma}

\begin{proof}
If $d^\epsilon(p^t) < \infty$ for every $p^t \in  S_{\FP}^\mathcal{T}$, then it is possible to reach the set of accepting states $F_{\FP}^{\mathcal{T}}$ from any state $p^t \notin  F_{\FP}^{\mathcal{T}}$  via a finite number of $\epsilon$-stochastic transitions. For any such state $p^t \notin  F_{\FP}^\mathcal{T}$, the policy $\pi^{\epsilon}_{GO}$ in \eqref{goeq} selects actions that drive the system to the set of accepting states over a shortest path under the $\epsilon$-stochastic transitions. Accordingly, under  $\pi^{\epsilon}_{GO}$, each action reduces the distance to the set of accepting states by one with probability at least $1-\epsilon$. In the remainder of proof, we refer to such transitions as ``intended transitions".  Furthermore, due to Assumption \ref{assump:distance1}, the distance to the set of accepting states can increase at most by one under any feasible transition that may happen with the remaining probability (``unintended transitions"). 
Note that observing at most $i_{max}$  unintended transitions within the next $k \geq i_{max}$ transitions ensures that the system reaches $F_{\FP}^\mathcal{T}$ as long as $i_{max}\leq k-i_{max}-d^\epsilon(p^t)$, i.e.,  $i_{max}\leq  \frac{k-d^\epsilon(p^t)}{2}$.
Moreover, for any  $i \leq i_{max}$, the number of all possible $k$-length sequences involving $i$ unintended transitions and $k-i$ intended transitions is $\frac{k!}{(k-i)!i!}$.
%of transition-types (intended or unintended) that involve $i$ unintended transitions and $k-i$ intended transitions is $\frac{k!}{(k-i)!i!}$.
% \begin{equation}
%   \frac{k!}{(k-i)!i!}.
% \label{num-seq}
% \end{equation}
Accordingly, since each intended transition occurs with probability at least $1-\epsilon$, for any $i_{max}\leq k$, the probability of observing a sequence of $k$ transitions involving at least $k-i_{max}$ intended transitions (at most $i_{max}$ unintended transitions) is lower bounded by
\begin{equation}
\sum\limits_{i=0}^{i_{max}}  \frac{k!}{(k-i)!i!} \epsilon^i (1-\epsilon)^{k-i}.
%   \frac{k!}{(k-i)!i!} \epsilon^i (1-\epsilon)^{k-i}.
\label{pr-seq}
\end{equation}

By setting $i_{max}$ equal to the largest possible value ensuring convergence to $F_{\FP}^\mathcal{T}$, i.e., $i_{max}= \lfloor \frac{k-d^\epsilon(p^t)}{2}\rfloor$, we obtain \eqref{lower-bound} for every $p^t \in S_{\FP}^\mathcal{T}$ such that $k \geq d^\epsilon(p^t)$.
\end{proof}

In light of Lemma~\ref{lemma}, we propose Alg.~1 for the construction of the time-product MDP and pruning of the feasible actions at each time-product MDP state. Overall, if an action taken at a time-product MDP state $p^t$ may result in a transition to another state $r^{t+1}$ such that \eqref{lower-bound} does not hold for $r^{t+1}$ or the remaining episode time is smaller than  $d^\epsilon(r^{t+1})$, then Alg.~1 removes that action from the feasible action set of $p^t$.  

 \begin{algorithm}[htb!]
 \label{alg1}
 \begin{center}
\resizebox{\columnwidth}{!}{
\begin{tabular}{ll}
%\hline 
\bf{Alg. 1:} \textbf{Offline construction of the pruned time-product MDP} \\
\hline
 \emph{Input:} $T$ (episode length), $\epsilon$ (motion uncertainty according to Assumption~\ref{assm})\\
 \emph{Input:}  $M$ (MDP), $\Phi$ (TL task), $Pr_{des}$ (desired satisfaction probability)\\ 
\emph{Output:}  $\mathcal{P}^{\mathcal{T}}$ (pruned time-product MDP)\\
\hline 
\mbox{\small $\;1:\;$}Create FSA of $\phi$, $\mathcal{A}  = (\mathcal{Q},q_{init},AP,\delta,F_{\mathcal{A} })$;  \\
\mbox{\small $\;2:\;$}Create product MDP, $ \mathcal{P} = M \times \mathcal{A}=(S_{\mathcal{P}}, P_{init}, A, \Delta_\FP, R_{\mathcal{P}}, F_{\mathcal{P}})$; \\ 
\mbox{\small $\;3:\;$}Create time-product MDP, $ \mathcal{P}^{\mathcal{T}} = \mathcal{P} \times \{0,\dots,T\}=(S^{\mathcal{T}}_{\mathcal{P}}, P^{\mathcal{T}}_{init}, A, \Delta_\FP^{\mathcal{T}}, R^{\mathcal{T}}_{\mathcal{P}}, F^{\mathcal{T}}_{\mathcal{P}})$; \\ 
\mbox{\small $\;4:\;$}Calculate the distance-to-$F_{\mathcal{P}}$, i.e., $d^\epsilon(p^t)$ for all $p^t \in S_{\mathcal{P}}^\mathcal{T}$ based on \eqref{eq:distance};\\
\mbox{\small $\;5:\;$}\textbf{Initialization:} $Act(p^t)=A$ for all $p^t \in S^\mathcal{T}_\mathcal{P}$ \\
\mbox{\small $\;6:\;$}\hspace{0.1cm}\textbf{for} each non-accepting state $p^t \in S_\mathcal{P}^\mathcal{T} \setminus F_\mathcal{P}^{\mathcal{T}}$ \\
\mbox{\small $\;7:\;$}\hspace{0.5cm}\textbf{for} each action $a \in Act(p^t)$ and $t \in\{0,\dots,T-1\}$\\
\mbox{\small $\;8:\;$}\hspace{0.9cm}Find $N (p^t,a)=\{(r,t+1)|\Delta_\FP(p,a,r)>0\}$ (states reachable from $p^t$ under $a$); \\
\mbox{\small $\;9:\;$}\hspace{0.9cm}$d_{max}= \max\limits_{r^{t+1} \in N(p^t,a)} d^\epsilon(r^{t+1})$; \\
\mbox{\small $\;10:\;$}\hspace{0.75cm}$k=T-t$ (the remaining episode time); \\
\mbox{\small $\;11:\;$}\hspace{0.75cm}$i_{max}= \lfloor \frac{k-1-d_{max}}{2}\rfloor$; \\
\mbox{\small $\;12:\;$}\hspace{0.75cm}\textbf{if } $i_{max} < 0$ or $\sum\limits_{i=0}^{i_{max}}  \frac{(k-1)!}{(k-1-i)!i!} \epsilon^i (1-\epsilon)^{k-1-i}<Pr_{des}$ \\
\mbox{\small $\;13:\;$}\hspace{1.15cm}$Act(p^t)=Act(p^t)\setminus \{a\}$ ;\\
\mbox{\small $\;14:\;$}\hspace{0.75cm}\textbf{end if}\\
\mbox{\small $\;15:\;$}\hspace{0.4cm}\textbf{end for}\\
%\mbox{\small $\;16:\;$}\hspace{0.4cm}\textbf{if} $Act(p^t)=\emptyset$\\
%\mbox{\small $\;17:\;$}\hspace{0.75cm} $Act(p^t)=\{ \pi^\epsilon_{GO}(p^t)\}$;\\
%\arg\min\limits_{a \in A} d^\epsilon(r^{t+1})$ such that $\Delta^\mathcal{T}(p^t,a,r^{t+1}) \geq 1-\epsilon$;\\
%\mbox{\small $\;18:\;$}\hspace{0.4cm}\textbf{end if} \\
\mbox{\small $\;16:\;$}\hspace{0.1cm}\textbf{end for}\\
\mbox{\small $\;17:\;$}$ \mathcal{P}^{\mathcal{T}} =(S^{\mathcal{T}}_{\mathcal{P}}, P^{\mathcal{T}}_{init}, Act:S_\FP^\mathcal{T} \rightarrow 2^A, \Delta_\FP^{\mathcal{T}}, R^{\mathcal{T}}_{\mathcal{P}}, F^{\mathcal{T}}_{\mathcal{P}})$;
\end{tabular}}
\end{center}
\end{algorithm}

%temporal logic constraint satisfaction during the episodes of Q-learning with probabilistic guarantees. 

Algorithm~1 is executed offline and its inputs are MDP $M$, the TL constraint $\phi$, the desired probability of satisfaction $Pr_{des}$, episode length $T$ calculated from the time bound of $\phi$, and the algorithm parameter $\epsilon$ which is a conservative bound of the motion uncertainty according to Assumption~\ref{assm}. Algorithm~1 starts with the construction of the FSA, then the product MDP, and then the time-product MDP (lines 1-3). For each time-product MDP state, the distance-to-$F_{\FP}$ is calculated based on \eqref{eq:distance} (line 4). Then, the feasible action set $Act(.)$ at each time-product MDP state is initialized with, $A$, the action set of MDP (line 5). Note that some actions in $A$ at particular states should not be taken to ensure the TL satisfaction. For example, if the constraint has not been satisfied yet and the remaining episode time is small, then actions leading to progress to the satisfaction should be selected rather than doing random exploration. For this reason, lines 6-16 are executed to prune the action sets to ensure the probabilistic satisfaction of $\phi$.
At each non-accepting state $p^t$ and for each action $a$ that can be taken at $p^t$, first the set of states that can be reached from $p^t$ under $a$ are found (line 8). Then, the maximum distance-to-$F_\FP$ is computed (line 9). This mainly captures the worst case (the furthest distance to the satisfaction) after taking $a$. Line 10 computes the number of actions $k$ that can be taken within the remaining episode time (including the current action selection). Line 11 calculates $i_{max}$ based on $k$ and $d_{max}$. Note that if $i_{max}<0$, then the distance-to-$F_\FP$ is greater than the number of actions that can be taken in the next time step $k-1$, which means that there is no way to satisfy the constraint. Thus, the algorithm considers the worst case situation (increasing the distance to $F_\FP$ by one) after taking action $a$. If there is insufficient time to satisfy the constraint ($i_{max}<0$) or the lower bound satisfaction probability (Lemma~\ref{lemma}) at the worst case state is less than $Pr_{des}$, then $a$ is pruned from $Act(p^t)$ (line 13). After this pruning routine is done for each non-accepting state, the pruned time-product MDP is constructed in line 17. 

\color{black}
Finally, we propose Alg.~2 which is a modified $Q$-learning algorithm executed over the pruned time-product MDP. For each time-product MDP state $z$, if the feasible action set $Act(z)$ is empty, then an action is selected according to the policy $\pi^\epsilon_{GO}$ in order to do progress to the constraint satisfaction (line 5). If $Act(z)$ is not empty, then an action is selected from $Act(z)$ (line 7). The general steps for the $Q$-updates are achieved in lines 9-12. When an episode ends, the next episode starts at the current physical state, but the automaton state and time are initialized as in line 14.

%% %\vspace{-2mm}
\begin{algorithm}[htb!]
\label{alg2}
\begin{center}
\resizebox{\columnwidth}{!}{
\begin{tabular}{ll}
%\hline 
\bf{Alg. 2:} \textbf{Probabilistically Guaranteed Constraint Satisfaction During Q-Learning} \\
\hline
 \emph{Input:} Time-product MDP $\mathcal{P}^{\mathcal{T}} =(S^{\mathcal{T}}_{\mathcal{P}}, P^{\mathcal{T}}_{init}, Act:S_\FP^\mathcal{T} \rightarrow 2^A, \Delta_\FP^{\mathcal{T}}, R^{\mathcal{T}}_{\mathcal{P}}, F^{\mathcal{T}}_{\mathcal{P}})$ \\
  \emph{Input:} Initial MDP state $s_{init}$ \\
\emph{Output:}  $\pi:S^\mathcal{T}_\mathcal{P} \rightarrow A$ (policy maximizing the sum of rewards under TL constraint)\\
\hline 
\mbox{\small $\;1:\;$}\textbf{Initialization:} Initial $Q-$table, $z=(s,q_{init},0) \in P_{init}^{\mathcal{T}}$ ;  \\
\mbox{\small $\;2:\;$}\hspace{0.05cm}\textbf{for}\hspace{0.1cm}  $j=0:N_{episode}$ \\
%\mbox{\small $\;3:\;$}\hspace{0.5cm}$p_i^t = (p_{init}, 0)$; \\
\mbox{\small $\;3:\;$}\hspace{0.5cm}\textbf{for} $t=0:T-1$ \\ 
\mbox{\small $\;4:\;$}\hspace{0.9cm}\textbf{if} $Act(z)=\emptyset$\\
\mbox{\small $\;5:\;$}\hspace{1.4cm} $a= \pi^\epsilon_{GO}(z)$;\\
\mbox{\small $\;6:\;$}\hspace{0.9cm}\textbf{else}\\
\mbox{\small $\;7:\;$}\hspace{1.4cm}Select an action \emph{a} from $Act(z)$ via $\epsilon-$greedy or $\pi$;\\
\mbox{\small $\;8:\;$}\hspace{0.9cm}\textbf{end if} \\
\mbox{\small $\;9:\;$}\hspace{0.9cm}Take action \emph{a}, observe the next state $z^\prime=(s^\prime,q^\prime, t+1)$ and reward \emph{r};\\
\mbox{\small $\;10:\;$}\hspace{0.75cm}$Q(z,a) = (1-\alpha_{ep}) Q(z,a) + \alpha_{ep} \big[ r + \gamma \max\limits_{a^\prime}  Q(z^\prime,a^{\prime}) \big]$; \\
\mbox{\small $\;11:\;$}\hspace{0.75cm}$\pi(z) = \arg\max\limits_a Q(z,a))$; \\
\mbox{\small $\;12:\;$}\hspace{0.75cm}$ z = z^\prime$; \\
\mbox{\small $\;13:\;$}\hspace{0.35cm}\textbf{end for}\\
\mbox{\small $\;14:\;$}\hspace{0.35cm}$ z = (s^\prime,q_{init},0)$; \\
\mbox{\small $\;15:\;$}\textbf{end for} 
\end{tabular}}
\end{center}
\end{algorithm}

 \vspace{-3mm}
\begin{theorem}[Constraint Satisfaction]
Given an MDP and some $\epsilon \in [0,1)$, let Assumptions~\ref{assm} and \ref{assump:distance1} hold. Let $T$ be the time bound of the TWTL constraint $\phi$ that should be satisfied with a probability of at least $Pr_{des}$ in each episode. If the set of initial states of the time-product MDP, i.e., $P_{init}^\mathcal{T}$, satisfies
\begin{equation}
    \label{eq:th1}
    \sum\limits_{i=0}^{\lfloor \frac{T-d^\epsilon(z)}{2}\rfloor}  \frac{T!}{(T-i)!i!} \epsilon^i (1-\epsilon)^{T-i} \geq Pr_{des}, \; \forall z \in P_{init}^\mathcal{T},
\end{equation}
then $\Pr\big(\mathbf{o}(jT),\mathbf{o}(jT+1),\dots,\mathbf{o}(jT+T) \models \phi{(\tau_j)}\big) \geq Pr_{des}$ for all $j \geq 0$, where $\mathbf{o}(jT),\mathbf{o}(jT+1),\dots,\mathbf{o}(jT+T)$ is the output word in episode $j$.
\end{theorem}

\begin{proof}
There are in total $T$ actions to be taken in each episode of length $T$. Under Alg.~2, each action is taken either from the set $Act$ (line 7 of Alg.~2) or by following the policy $\pi^\epsilon_{GO}$ (line 5 of Alg.~2). Accordingly,  such sequences of $T$ actions can be grouped into three disjoint sets: 1) sequences such that the last action taken at $T-1$ is selected from the set $Act$, 2) sequences such that the action at $t$ is selected from the set $Act$, and all the following actions are taken according to the policy $\pi^\epsilon_{GO}$ for some $t \in \{0,1, \hdots, T-2\}$, 3) all the actions are taken according to the policy $\pi^\epsilon_{GO}$. These three sets are also illustrated in Fig.~\ref{fig:all_cases} as different cases. We will now show that the probability of constraint satisfaction is at least $Pr_{des}$ in each of these three cases. 
%\vspace{-2mm}
\begin{figure}[!h]
    \centering
\includegraphics[width=0.6\columnwidth, trim={0 0 0 1cm},clip]{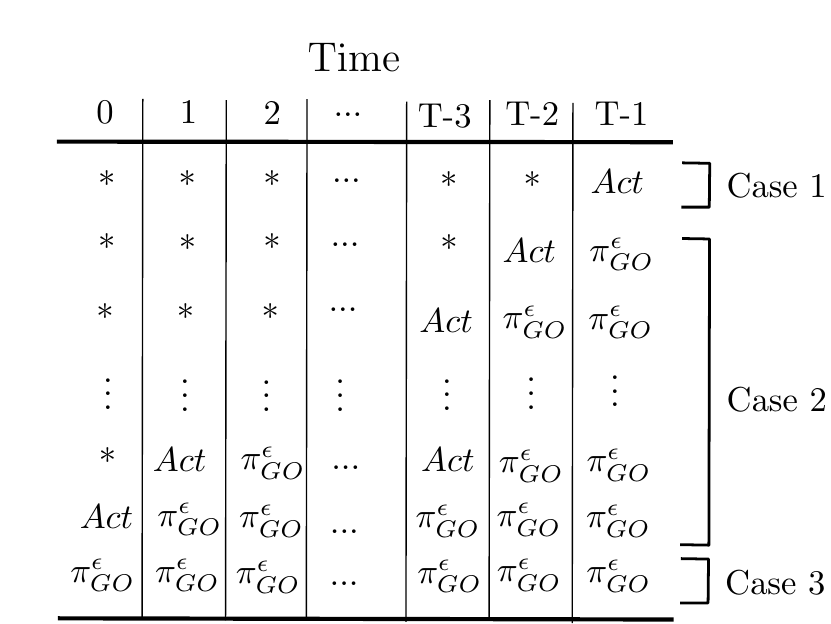}
    \caption{Possible sequences of actions generated in any episode via Alg.~2 are grouped into distinct sets based on whether the action at each time $t \in \{0,1, \hdots, T-1\}$ is taken from the set $Act$ (line 7 of Alg.~2) or by following the policy $\pi^\epsilon_{GO}$ (line 5 of Alg.~2). The symbol $*$ is used to allow for both possibilities, i.e., the action may be taken from $Act$ or via $\pi^\epsilon_{GO}$.}
    \label{fig:all_cases}
\end{figure}
%\vspace{-2mm}

\noindent\textbf{Case 1}: Consider an arbitrary action sequence such that the last action taken at $T-1$ is selected from the set $Act$ according to line 7 of Alg.~2. This means that the system's time-product MDP state at time $T-1$, $p^{T-1}$, was such that $Act(p^{T-1}) \neq \emptyset$. Note that $Act(p^{T-1})$ only contains actions for which the condition in line 12 of Alg.~1 is false (hence the pruning in line 13 is not executed). Accordingly, we will show that if the last action is taken from the set $Act(p^{T-1})$, then the system must reach an accepting state. We will prove this statement by contradiction: suppose that the last action $a \in Act(p^{T-1})$ but it results in a non-accepting state which has a positive finite distance $d^\prime$ to the set of accepting states. By definition, the maximum distance to the accepting states that can result from taking action $a$ at state  $p^{T-1}$, i.e., $d_{max}$ computed in line 9 of Alg.~1, satisfies $d_{max} \geq d^\prime>0$. Since the remaining episode time is $k=1$, $i_{max}$ in line 11 of Alg.~1 becomes a negative number for any $d_{max}>0$. If $i_{max}<0$, then the condition in line 12 is true and the action $a$ must have been pruned from $Act(p^{T-1})$, which is a contradiction. Consequently, whenever Alg.~2 generates an action sequence as in Case 1, the system is surely at an accepting state at the end of that episode, which implies the satisfaction of the constraint. 

%Now, the remaining episode time at the last action selection is $1$ ($k=1$ in line 10 Alg.~1). If $k=1$, $i_{max}$ becomes negative for any values of $d_{}$

%whose execution will always result in being at an accepting state
%Given the system's current state at $T-1$, i.e., $p^{T-1}$, if there exists an action $a^\prime$ whose execution will always result in being at an accepting state, then the corresponding $d_{max}=0$ (line 9 in Alg.~1). Since $t=T-1$, the remaining episode time is $k=1$ (line 10). Thus, $i_{max}$ becomes $0$ (line 11). If $k$ and $i_{max}$ are plugged into the expression in line 12 in Alg.~1 and $Pr_{des}<1-\epsilon$, then line 13 is skipped, and $a^\prime \in Act(p^{T-1})$. If $Act(p^{T-1}) \neq \emptyset$, line 17 is skipped and the system will select an action from $Act(p^{T-1})$, which only contains actions that certainly drive the system to an accepting state at time $T$. Hence, if $Act(p^{T-1})$ is not defined according to line 17, then the system will satisfy the specification starting from any $p^0$. If $Pr_{des} \geq 1-\epsilon$, $a^\prime \notin Act(p^{T-1})$ and if all actions are pruned from $Act(p^{T-1})$, then Alg.~1 line 17 assigns an action to $Act(p^{T-1})$ according to the policy $\pi^\epsilon_{GO}$, which will be discussed in the next case.  

%\noindent\textbf{Case 2}: Let $Act(p^{T-1})= \emptyset$, i.e.,  all actions are pruned in lines 7-15 in Alg.~1 at $t=T-1$. Then, as per line 17 in Alg.~1, $Act(p^{T-1})$ contains only one action according to the policy $\pi^\epsilon_{GO}$. For these 

\noindent\textbf{Case 2}: Consider an arbitrary action sequence such that for some $t^\prime \in \{0,1,2,\dots,T-2\}$, the action at time  $t^\prime$ is  selected  from the set $Act(p^{t'})$, and all the following actions (i.e., for all $t \in \{t^\prime+1,\dots,T-1\}$) are taken according to the policy $\pi^\epsilon_{GO}$.
%Suppose that $\pi^\epsilon_{GO}$ is used in the definitions of $Act(p^{t^\prime})$ for all $t^\prime=\{t^*,\dots,T-1\}$, implying that $Act(p^{t^*-1})$ is not defined according to the policy $\pi^\epsilon_{GO}$. 
Accordingly, taking any action $a^\prime \in Act(p^{t^\prime})$ at time $t'$ ensures that $\Pr\big(p^{t^\prime+1} \xrightarrow{T-t^\prime-1} F_\FP^{\mathcal{T}};\pi^{\epsilon}_{GO} \big) \geq Pr_{des}$, i.e., constantly following   $\pi^{\epsilon}_{GO}$ in the remaining time steps would drive the system to the set of accepting states with probability at least  $Pr_{des}$ (if this is not true, then line 12 would have been executed and $a^\prime$ would have been pruned). Overall, this means that the probability of generating an accepting output word starting with a prefix $p^0,\dots,p^{t^\prime}$ and then by purely following the policy $\pi^\epsilon_{GO}$ from $t^\prime+1$ to $T-1$ is greater than $Pr_{des}$. Consequently, whenever Alg.~2 generates an action sequence as in Case 2, the system reaches an accepting state by the end of that episode with probability at least $Pr_{des}$. 

\noindent\textbf{Case 3}: In this case, each action is taken according to the policy $\pi^\epsilon_{GO}$ for all $t=\{0,1,\dots,T-1\}$. In light of Lemma~\ref{lemma}, if \eqref{eq:th1} is true, then $\Pr\big(p^0 \xrightarrow{T} F_\FP^{\mathcal{T}};\pi^{\epsilon}_{GO} \big) \geq Pr_{des}$ for all $p^0 \in S_\mathcal{P}^\mathcal{T}$, which means that the probability of reaching an accepting state (or satisfying the constraint) by purely following the policy $\pi^\epsilon_{GO}$ is at least $Pr_{des}$.

 Cases 1, 2, 3 are disjoint and cover all the possible outcomes under Alg.~2. Since the probability of constraint satisfaction in each case is at least $Pr_{des}$, we conclude that the probability of constraint satisfaction is at least $Pr_{des}$ in each episode, i.e., $\Pr\big(\mathbf{o}(jT),\mathbf{o}(jT+1),\dots,\mathbf{o}(jT+T) \models \phi{(\tau_j)}\big) \geq Pr_{des}$ for all $j \geq 0$, where $\mathbf{o}(jT),\mathbf{o}(jT+1),\dots,\mathbf{o}(jT+T)$ is the output word in episode $j$.
\end{proof}
\begin{remark} [Optimality]
Given a finite MDP using $Q$-learning, if each action is repetitively implemented in each state for infinite number of times and the learning rate $\alpha$ decays appropriately, then the $Q$-values converge to the optimal $Q$-values with probability $1$ \cite{watkins1992}. In this paper, we use $Q$-learning over a finite time-product-MDP (the update rule is in line 10 in Alg.~2), and the $Q$-values of the time-product-MDP states will converge to the optimal $Q$-values if line 10 in Alg.~2 is updated infinitely many times.
\end{remark}

\begin{figure*}[ht]
    \centering
    %% %\vspace{0.3cm}
    \hspace*{-4cm}
    \subfigure[]{\includegraphics[width=0.16\textwidth]{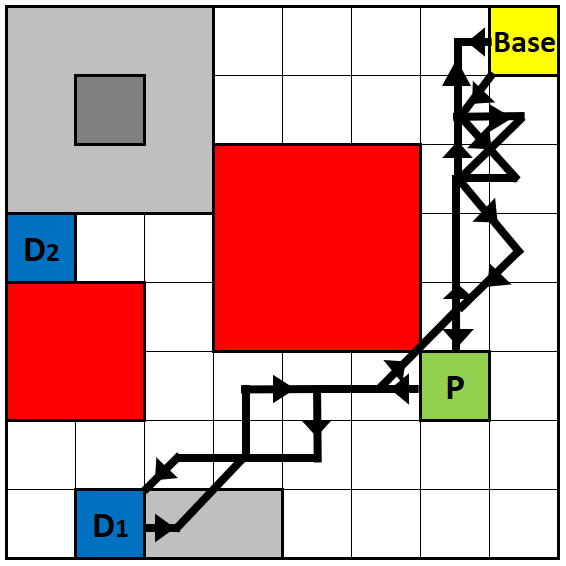}}
    \hspace{1cm}
    \subfigure[]{\includegraphics[width=0.16\textwidth]{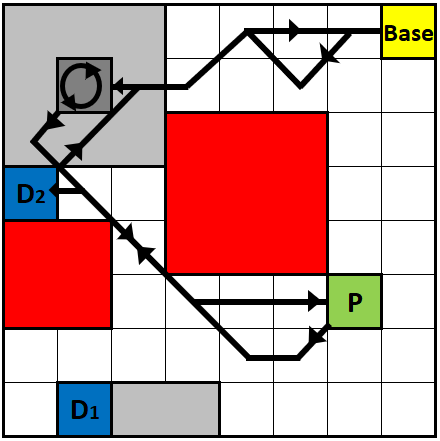}}
    \hspace{1cm}
    \subfigure[]{\includegraphics[width=0.16\textwidth]{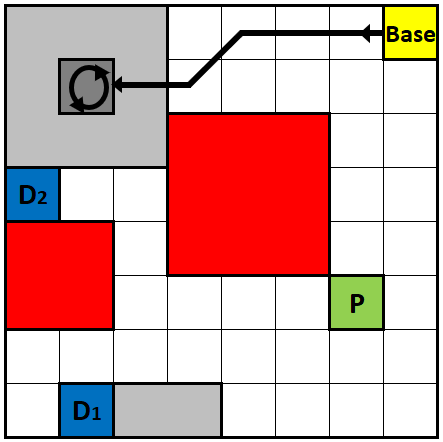}} 
    \hspace{1cm}
    \subfigure[]{\includegraphics[width=0.16\textwidth]{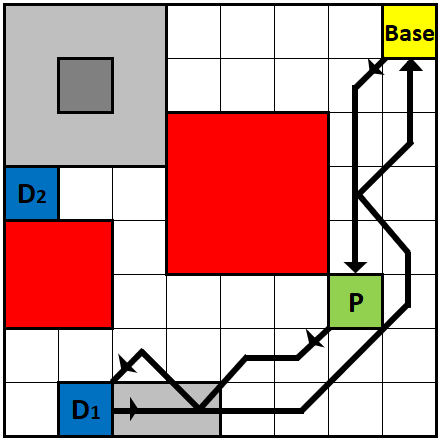}}
    \hspace*{-4cm}
    \caption{A grid environment where  yellow, green, blue, red cells  respectively represent the base station, the pick-up region, the delivery regions, and the obstacles. The gray cells (darker shade indicates a higher reward) are the regions in which monitoring is rewarded. The black arrows denote some sample trajectories, which are learned in three cases with different goals: (a) proposed algorithm after the first episode, (b) proposed algorithm at the end of the learning, (c) only maximizing the expected reward, and (d) only satisfying the TWTL specification.}
    \label{fig:sim_benchmarks}
    % %\vspace{-4mm}
\end{figure*}

% \begin{figure*}[ht]
% \begin{center}
% \includegraphics[width=0.95\textwidth]{RAL_eps2.JPG}%
% \caption{Weighted discounted rewards of Cases 1, 2, and 3 from left to right. The results are smoothed over 10000, 500, and 20000 episodes respectively.}
% \label{fig:sim2}
% %% %\vspace{-0.2cm}
% \end{center}
% \end{figure*}

\section{Simulation Results}
In this section, we present some case studies implemented on Python 2.7, and performed on a PC with an Intel i7-7700HQ CPU at 2.8 GHz processor and 16.0 GB RAM. In these case studies, we consider an agent moving over an $8 \times 8$ grid as shown in Fig.~\ref{fig:sim_benchmarks} and selecting actions from the set $A = \{N, NE, E, SE, S, SW, W, NW, Stay\}$. Under these actions, the agent can maintain its current position or move to any of the feasible neighboring cells (including those in the ordinal directions). 
 Each action leads to the intended transition with probability $0.9$ or to one of the other feasible transitions (selected uniformly at random) with probability $0.1$. For example, if the agent takes the action $NE$, it moves to the neighboring cell in the direction $NE$ with probability 
$0.9$ or, with probability  $0.1$, it stays at its location or moves to a feasible neighboring cell in any of the other 7 directions. 

%  In this TWTL specification, returning to the base is included to address the energy limitation of the agent (e.g., for refueling or recharging).

We consider a scenario where the agent is required to periodically perform a pickup and delivery task while maximizing situational awareness by collecting measurements from the environment.
 Accordingly, the reward $r_t$  represents the value of monitoring the agent's current position on the grid and the discount factor in \eqref{pistareq} is selected as  $\gamma = 0.95$.
In Fig.~\ref{fig:sim_benchmarks}, the light gray cells, the dark gray cell, and all the other cells yield a reward of $1$, $10$, and $0$, respectively. The pickup and delivery task is encoded as a TWTL constraint: $[H^1 P]^{[0,20]} . ([H^1 D_1]^{[0,20]} \; \vee \; [H^1 D_2]^{[0,20]}) . [H^1 Base]^{[0,20]}$, which means that \textit{"go to the pickup location $P$ and stay there for $1$ time step in the first $20$ time steps and immediately after that go to one of the delivery locations, $D_1$ or  $D_2$, and stay there for $1$ time step within $20$ time steps, and immediately after that go to $Base$ and stay there for $1$ time step within $20$ time steps."}. Based on the time bound of this TWTL specification, the length of each episode is selected as $62$ time steps.

%Overall, the agent's goal is to learn a control policy that maximizes the expected sum of rewards while satisfying the TWTL constraint in each episode with a probability at least $Pr_{des}$. 

We present our simulation results under five cases. The first three cases are performed to compare the performance under three different behaviors: 1) maximizing the expected reward under the TWTL constraint (proposed approach), 2) maximizing the expected reward without any constraint, and 3) learning to satisfy the TWTL formula. %We show the corresponding learning curves in Fig.~\ref{fig:sim2}. 
The fourth case is performed to demonstrate how the performance under the proposed approach changes when the desired probability of constraint satisfaction, $Pr_{des}$, increases or the probabilities of likely transitions are underestimated, i.e., Alg.~2 is executed in this scenario with some $\epsilon>0.1$. Finally, the fifth case is performed to demonstrate the scalability of the proposed algorithm under varying sizes of the time-product MDP.

 {In Case 1, $Pr_{des}$ is chosen as $0.7$ and Alg.~2 is executed by using $\epsilon=0.1$ for $400000$ episodes. The TWTL constraint is satisfied in $341060$ episodes implying a success ratio $0.853 > Pr_{des} = 0.7$. The run time of the algorithm is $537.3 $ seconds. We show a sample trajectory after the first episode and $400000$ episodes in Fig.~\ref{fig:sim_benchmarks} (a) and (b), respectively. Note that the agent satisfies the constraint by delivering to $D_1$ in the first episode. However, the agent eventually learns that it can collect more rewards by delivering to $D_2$ instead of $D_1$.}

\begin{table*}[t!]
  \centering
  \resizebox{\textwidth}{!}{
  \begin{tabular}{|c |c |c |c |c |c |c |c |c |c |c|}%{|P{0.7cm}|P{0.9cm}|P{0.9cm}|P{0.9cm}|P{0.9cm}|P{0.9cm}|P{0.9cm}|P{0.9cm}|P{0.9cm}|P{0.9cm}|P{0.95cm}|}
    \hline
  $(\epsilon_{est}, Pr_{des})$& (0.1, 0.5) & (0.1, 0.6) & (0.1, 0.7) &  (0.15, 0.5) & (0.15, 0.6) & (0.15, 0.7) & (0.2, 0.5) & (0.2, 0.6) & (0.2, 0.7)  \\ \hline
    Success Ratio [\%]& 54.1& 69.0& 83.9& 75.8& 90.2 &95.5 &93.6 &98.1 &99.0
 \\ \hline
    Avg. Rewards at the Last 5000 Episodes & 205.8
 & 195.9 & 193.9& 175.4 &176.7 &155.5 &141.6 &144.1 &130.6
 \\ \hline
  \end{tabular}}
  \caption{Simulation results for the task $[H^1 P]^{[0,20]} . ([H^1 D_1]^{[0,20]} \; \vee \; [H^1 D_2]^{[0,20]}) . [H^1 Base]^{[0,20]}$ and the real action uncertainty of $\epsilon_{real}=0.1$}\label{tab_all}
   %\vspace{-8mm}
\end{table*}

In Case 2, we remove the TWTL constraint and demonstrate the performance when the agent follows the standard Q-learning. This case is simulated for $15 000$ episodes with a run time of $5$ seconds. As shown in Fig.~\ref{fig:sim_benchmarks}(c), the agent simply learns to quickly go to the highest reward zone.% as quickly as possible.

In Case 3, we use the standard Q-learning to learn satisfying the desired TWTL task used in Case 1, without any consideration of the monitoring performance. To accomplish this, we assign rewards to the accepting states of the time-product MDP (the reward at any other state is zero). While the agent eventually learns to satisfy the TWTL constraint by following a shortest path as shown in Fig.~\ref{fig:sim_benchmarks}(d), there is no guarantee on the constraint satisfaction in the early stages of learning. In this case, the constraint is satisfied in only $143 312$ episodes, which is $35.8\% $ of the total number of episodes. The total run time for this case is $199.8$ seconds.

In Case 4, we investigate how the parameters $\epsilon$ and $Pr_{des}$ influence the performance. To this end, the proposed algorithm is executed under varying values of $Pr_{des}=0.5, 0.6, 0.7$ and $\epsilon = 0.1, 0.15, 0.2$. In this case, $\epsilon>0.1$ indicates a conservative estimation of uncertainty. For example, for $\epsilon = 0.2$, the proposed algorithm assumes that the agent can move in its intended direction with a probability at least $0.8$ whereas such transitions actually occur with probability 0.9. The results for Case 4 are depicted in Table~\ref{tab_all}. For a fixed $Pr_{des}$, we observe that increasing $\epsilon$ reduces collected reward and increases the probability of constraint satisfaction. This is due to the fact that overestimating the uncertainty makes the algorithm overly cautious in exploration and more inclined to first satisfy the TWTL constraint. Moreover, for a fixed $\epsilon$, increasing $Pr_{des}$ reduces the collected reward since the constraint becomes more restrictive.    

%\color{red}
In Case 5, we consider the same structure of the TWTL task but with different time windows, i.e., we create three scenarios considering the TWTL specification $[H^1 P]^{[0,k]} . ([H^1 D_1]^{[0,k]} \; \vee \; [H^1 D_2]^{[0,k]}) . [H^1 Base]^{[0,k]}$ with $k=20, 30, 40$). In all scenarios, the product-MDPs have $154$ states since the structure of the relaxed automaton is independent from the time windows. However, the sizes of the time-product MDPs vary since the lengths of episodes are determined by the time windows ($62, 92, 122$ time steps). The time-product MDPs have $9548, 14168, 18788$ states for $k= 20, 30, 40$, respectively, and their offline construction took approximately $3.44, 8.65, 22.66$ seconds.  The algorithm parameters are selected as $\epsilon=0.15$ and $Pr_{des}=0.7$. The success ratios in these three scenarios are observed as $0.95, 0.97, 0.97$ which are all greater than $Pr_{des} = 0.85$. This difference between the success ratios and $Pr_{des}$ is mainly due to the overestimation of uncertainty, i.e., $\epsilon=0.15>0.1$. The average rewards in the last 5000 episodes are observed as $155.04, 361.99, 582.39$, respectively. As the time window of the TWTL task (hence the episode length) increases, the agent has more time in each episode to explore the environment so the collected reward increases.

\section{Conclusion}
This paper proposes a constrained reinforcement learning algorithm for maximizing the expected sum of rewards in a Markov Decision Process (MDP) while satisfying  a bounded temporal logic constraint in each episode with a desired probability. We represent the bounded temporal logic constraint as a finite state automaton. We then construct a time-product MDP and formulate a constrained reinforcement learning problem. We derive a lower bound on the probability of satisfying the constraint from each state of the time-product MDP in the remaining episode time. This lower bound is computed by using some limited knowledge on which transitions are sufficiently likely in the system. The proposed approach uses this lower bound to keep the probability of constraint satisfaction above the desired threshold by restricting the actions that can be taken during learning. %The proposed approach is also demonstrated via simulations. 

As a future direction, we plan to 1) explore how similar guarantees on the probability of satisfaction of temporal logic constraints during learning can be achieved by multi-agent systems in a distributed manner (e.g., \cite{Yazicioglu17TCNS,Bhat19,Peterson20}), and 2) extend our methods to dynamic environments, where the constraint satisfaction requires reaching the accepting states on a time-varying graph (time-product MDP) (e.g., \cite{yaziciouglu2020}).

\color{black}
\bibliography{ref}

\end{document}